\newif\ifJOURNAL
\JOURNALfalse
\newif\ifCONF
\CONFfalse
\newif\ifarXiv
\arXivfalse
\newif\ifWP
\WPfalse
\newif\ifFULL
\FULLfalse

\arXivtrue

\newif\ifnotCONF	
\notCONFtrue
\ifCONF\notCONFfalse\fi

\newif\ifnotarXiv	
\notarXivtrue
\ifarXiv\notarXivfalse\fi

\newif\ifTR		
\TRfalse
\ifarXiv\TRtrue\fi
\ifWP\TRtrue\fi
\ifFULL\TRtrue\fi
\newif\ifnotTR
\notTRtrue
\ifarXiv\notTRfalse\fi
\ifWP\notTRfalse\fi
\ifFULL\notTRfalse\fi

\ifCONF
  
\fi
\ifarXiv
  
\fi
\ifWP
  
\fi
\ifFULL
  
\fi

\newlength{\picturewidth}
\ifCONF
  \documentclass[wcp]{jmlr}
  \jmlrvolume{25}
  \jmlryear{2012}
  \jmlrworkshop{Asian Conference on Machine Learning}
  \newcommand{\Extra}[1]{}
\fi

\ifarXiv
  \documentclass[10pt]{article}
  \usepackage{amsmath,amsthm,amsfonts,amssymb,latexsym,graphicx,stmaryrd,natbib,url}
  \newcommand{\Extra}[1]{}
\fi

\ifWP
  \documentclass{article}
  \usepackage{amsmath,amsthm,amsfonts,amssymb,latexsym,epsfig,graphicx,stmaryrd,natbib,url}
  \input{OT2enc.def}
  
  \usepackage{CJK}
  \input{/Doc/Computing/Latex/kp.txt}
  \newcommand{\Extra}[1]{}
\fi

\ifFULL
  \documentclass{article}
  \usepackage{amsmath,amsthm,amsfonts,amssymb,latexsym,color,graphicx,stmaryrd,natbib,url}
  \newcommand{\Extra}[1]{\red{#1}}
  \newcommand{\red}[1]{\textcolor{red}{#1}}
  
  \newcommand{\bluebegin}{\begingroup\color{blue}}
  \newcommand{\blueend}{\endgroup}

\fi

\ifnotCONF
  \bibpunct{(}{)}{;}{a}{,}{,}  
\fi

\allowdisplaybreaks[4]

\DeclareMathOperator{\Prob}{\mathbb{P}}
\DeclareMathOperator{\bin}{bin}    
\DeclareMathOperator{\Bet}{Bet}    
\DeclareMathOperator{\overbin}{\overline{bin}}

\ifnotCONF
  \newtheorem{lemma}{Lemma}
  \newtheorem{proposition}{Proposition}

  \newtheorem*{remark}{Remark}
\fi

\ifCONF
  \title[Conditional validity of inductive conformal predictors]{Conditional validity of inductive conformal predictors}
  \author{\Name{Vladimir Vovk} \Email{v.vovk@rhul.ac.uk}\\
  \addr Computer Learning Research Centre,
    Department of Computer Science,
    Royal Holloway, University of London,
    Egham, Surrey TW20 0EX, UK
  }
  \editors{Steven C.~H.~Hoi and Wray Buntine}
\fi

\ifarXiv
  \title{Conditional validity of inductive conformal predictors}
  \author{Vladimir Vovk\\
  \texttt{v.vovk{\rm@}rhul.ac.uk}\\
  \url{http://vovk.net}}
\fi

\ifWP
  \title{Conditional validity of inductive conformal predictors}
  \author{Vladimir Vovk}
  
  \twodatestrue
  
\fi

\ifFULL
  \title{Conditional validity of inductive conformal predictors}
  \author{Vladimir Vovk\\
  \texttt{vovk{\rm@}cs.rhul.ac.uk}\\
  \url{http://vovk.net}}
\fi

\begin{document}
\maketitle

\begin{abstract}
  Conformal predictors are set predictors that are automatically valid
  in the sense of having coverage probability equal to or exceeding
  a given confidence level.
  Inductive conformal predictors are a computationally efficient version of conformal predictors
  satisfying the same property of validity.
  However, inductive conformal predictors have been only known to control
  unconditional coverage probability.
  This paper explores various versions of conditional validity
  and various ways to achieve them using inductive conformal predictors and their modifications.
\end{abstract}

\section{Introduction}
\label{sec:introduction}

This paper continues study of the method of conformal prediction\ifnotCONF,
  introduced in \citet{vovk/etal:1999-full} and \citet{saunders/etal:1999-full}
  and further developed in \citet{vovk/etal:2005book}\fi.
An advantage of the method is that its predictions
(which are set rather than point predictions)
automatically satisfy a finite-sample property of validity.
Its disadvantage is its relative computational inefficiency in many situations.
A modification of conformal predictors,
called inductive conformal predictors\ifnotCONF,
  was proposed in \citet{papadopoulos/etal:2002ICMLA-full,papadopoulos/etal:2002ECML-full}
  with the purpose of \fi
improving on the computational efficiency of conformal predictors.

Most of the literature on conformal prediction
studies the behavior of set predictors in the online mode of prediction,
perhaps because the property of validity can be stated in an especially strong form
in the on-line mode \ifnotCONF
 (as first shown in \citealt{vovk:2002FOCS-full})\fi.
The online mode, however, is much less popular in applications of machine learning
than the batch mode of prediction.
This paper follows the recent papers by \citet{lei/etal:2011}, \citet{lei/wasserman:2012},
and \citet{lei/etal:2012} studying properties of conformal prediction in the batch mode;
we, however, concentrate on inductive conformal prediction
(also considered in \citealt{lei/etal:2012}).
\ifnotCONF
  The performance of inductive conformal predictors in the batch mode
  is illustrated
  using the well-known \texttt{Spambase} data set;
  for earlier empirical studies of conformal prediction in the batch mode
  see, e.g., \citet{vanderlooy/etal:2007-full}.
  The conference version of this paper is published as \citet{vovk:2012ACML-short}.
\fi

We will usually be making the \emph{assumption of randomness},
which is standard in machine learning and nonparametric statistics:
the available data is a sequence of \emph{examples}
generated independently from the same probability distribution $P$.
(In some cases we will make the weaker assumption of exchangeability;
for some of our results even weaker assumptions,
such as conditional randomness or exchangeability, would have been sufficient.)
Each example consists of two components: an \emph{object} and a \emph{label}.
We are given a \emph{training set} of examples and a new object,
and our goal is to predict the label of the new object.
(If we have a whole \emph{test set} of new objects,
we can apply the procedure for predicting one new object to each of the objects
in the test set.)

The two desiderata for inductive conformal predictors are their validity and efficiency:
validity requires that the coverage probability of the prediction sets
should be at least equal to a preset confidence level,
and efficiency requires that the prediction sets should be as small as possible.
However, there is a wide variety of notions of validity,
since the ``coverage probability'' is, in general, conditional probability.
The simplest case is where we condition on the trivial $\sigma$-algebra,
i.e., the probability is in fact unconditional probability,
but several other notions of conditional validity are depicted in Figure~\ref{fig:cube},
where T refers to conditioning on the training set,
O to conditioning on the test object, and L to conditioning on the test label.
The arrows in Figure~\ref{fig:cube} lead from stronger to weaker notions of conditional validity;
U is the sink and TOL is the source (the latter is not shown).

\begin{figure}[tb]
  \begin{center}
    \input{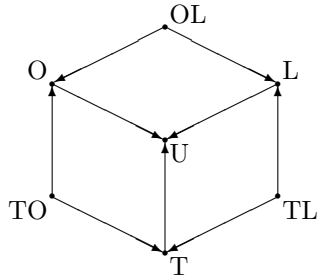}
  \end{center}

\vspace{-1mm}

\caption{Eight notions of conditional validity.
  The visible vertices of the cube are U (unconditional),
  T (training conditional),
  O (object conditional),
  L (label conditional),
  OL (example conditional),
  TL (training and label conditional),
  TO (training and object conditional).
  The invisible vertex is TOL (and corresponds to conditioning on everything).
\label{fig:cube}}
\end{figure}

Inductive conformal predictors will be defined in Section~\ref{sec:ICP}.
They are automatically valid, in the sense of unconditional validity.
It should be said that, in general, the unconditional error probability
is easier to deal with than conditional error probabilities;
e.g., the standard statistical methods of cross-validation and bootstrap
provide decent estimates of the unconditional error probability
but poor estimates for the training conditional error probability:
see \citet{hastie/etal:2009},
Section~7.12.

In Section~\ref{sec:T} we explore training conditional validity of inductive conformal predictors.
Our simple results (Propositions~\ref{prop:2-parameter} and ~\ref{prop:2-parameter-exact}) are of the PAC type,
involving two parameters:
the target training conditional coverage probability $1-\epsilon$
and the probability $1-\delta$ with which $1-\epsilon$ is attained.
They show that inductive conformal predictors achieve training conditional validity automatically
(whereas for other notions of conditional validity the method has to be modified).
We give self-contained proofs of Propositions~\ref{prop:2-parameter} and~\ref{prop:2-parameter-exact},
but Appendix~A 
explains how they can be deduced from classical results about tolerance regions.

In the following section, Section~\ref{sec:CICP}, we introduce a conditional version
of inductive conformal predictors and explain, in particular, how it achieves
label conditional validity.
Label conditional validity is important as it allows the learner to control
the set-prediction analogues of false positive and false negative rates.
Section~\ref{sec:O} is about object conditional validity
and its main result (a version of a lemma in \citealt{lei/wasserman:2012}) is negative:
precise object conditional validity cannot be achieved in a useful way
unless the test object has a positive probability.
Whereas precise object conditional validity is usually not achievable,
we should aim for approximate and asymptotic object conditional validity
when given enough data
(cf.\ \citealt{lei/wasserman:2012}).

Section~\ref{sec:experiments} reports on the results of empirical studies
for the standard \texttt{Spambase} data set
(see, e.g., \citealt{hastie/etal:2009}, Chapter 1, Example 1, and Section 9.1.2).
Section~\ref{sec:ROC} discusses close connections between an important class of ICPs
and ROC curves.
Section~\ref{sec:conclusion} concludes\ifnotCONF\
and Appendix~A discusses connections with the classical theory of tolerance regions
(in particular, it explains how Propositions~\ref{prop:2-parameter} and ~\ref{prop:2-parameter-exact}
can be deduced from classical results about tolerance regions)\fi.

\section{Inductive conformal predictors}
\label{sec:ICP}

The example space will be denoted $\mathbf{Z}$;
it is the Cartesian product $\mathbf{X}\times\mathbf{Y}$ of two measurable spaces,
the object space and the label space.
In other words, each example $z\in\mathbf{Z}$ consists of two components:
$z=(x,y)$, where $x\in\mathbf{X}$ is its object and $y\in\mathbf{Y}$ is its label.
Two important special cases are the problem of \emph{classification},
where $\mathbf{Y}$ is a finite set (equipped with the discrete $\sigma$-algebra),
and the problem of \emph{regression}, where $\mathbf{Y}=\mathbb{R}$.

Let $(z_1,\ldots,z_l)$ be the training set, $z_i=(x_i,y_i)\in\mathbf{Z}$.
We split it into two parts,
the \emph{proper training set} $(z_1,\ldots,z_m)$ of size $m<l$
and the \emph{calibration set} of size $l-m$.
An \emph{inductive conformity $m$-measure} is a measurable function
$A:\mathbf{Z}^m\times\mathbf{Z}\to\mathbb{R}$;
the idea behind the \emph{conformity score} $A((z_1,\ldots,z_m),z)$
is that it should measure how well $z$ conforms to the proper training set.
A standard choice is
\begin{equation}\label{eq:score}
  A((z_1,\ldots,z_m),(x,y))
  :=
  \Delta(y,f(x)),
\end{equation}
where $f:\mathbf{X}\to\mathbf{Y}'$ is a prediction rule found from $(z_1,\ldots,z_m)$
as the training set
and $\Delta:\mathbf{Y}\times\mathbf{Y}'\to\mathbb{R}$ is a measure of similarity between a label and a prediction.
Allowing $\mathbf{Y}'$ to be different from $\mathbf{Y}$
(often $\mathbf{Y}'\supset\mathbf{Y}$)
may be useful when the underlying prediction method gives additional information
to the predicted label;
e.g., the MART procedure used in Section~\ref{sec:experiments}
gives the logit of the predicted probability that the label is $1$.

\ifnotCONF
  \begin{remark}
    {\rm The idea behind the term ``calibration set'' is that this set allows us
    to calibrate the conformity scores for test examples
    by translating them into a probability-type scale.}
  \end{remark}
\fi

The \emph{inductive conformal predictor} (ICP) corresponding to $A$
is defined as the set predictor
\begin{equation}\label{eq:ICP}
  \Gamma^{\epsilon}(z_1,\ldots,z_l,x)
  :=
  \{y \mid p^y>\epsilon\},
\end{equation}
where $\epsilon\in[0,1]$ is the chosen \emph{significance level}
($1-\epsilon$ is known as the \emph{confidence level}),
the \emph{p-values} $p^y$, $y\in\mathbf{Y}$, are defined by
\begin{equation}\label{eq:p}
  p^y
  :=
  \frac
  {
    \left|\left\{
      i=m+1,\ldots,l \mid \alpha_i\le\alpha^y
    \right\}\right|
    +
    1
  }
  {l-m+1},
\end{equation}
and
\begin{equation}\label{eq:alphas}
  \alpha_i
  :=
  A((z_1,\ldots,z_m),z_i),
  \quad
  i=m+1,\ldots,l,
  \qquad
  \alpha^y
  :=
  A((z_1,\ldots,z_m),(x,y))
\end{equation}
are the conformity scores.
Given the training set and a new object $x$ the ICP predicts its label $y$;
it \emph{makes an error} if $y\notin\Gamma^{\epsilon}(z_1,\ldots,z_l,x)$.

The random variables whose realizations are $x_i$, $y_i$, $z_i$, $z$
will be denoted by the corresponding upper case letters
($X_i$, $Y_i$, $Z_i$, $Z$, respectively).
The following proposition of validity is almost obvious.

\begin{proposition}[\citealp{vovk/etal:2005book}, Proposition~4.1]
  \label{prop:validity-ICP}
  If random examples $Z_{m+1},\ldots,Z_l,$ $Z_{l+1}=(X_{l+1},Y_{l+1})$ are exchangeable  
  (i.e., their distribution is invariant under permutations),
  the probability of error $Y_{l+1}\notin\Gamma^{\epsilon}(Z_1,\ldots,Z_l,X_{l+1})$
  does not exceed $\epsilon$ for any $\epsilon$ and any inductive conformal predictor $\Gamma$.
\end{proposition}

In practice the probability of error is usually close to $\epsilon$
(as we will see in Section~\ref{sec:experiments}).

\section{Training conditional validity}
\label{sec:T}

As discussed in Section~\ref{sec:introduction},
the property of validity of inductive conformal predictors is unconditional.
The property of conditional validity can be formalized using a PAC-type 2-parameter definition.
It will be convenient to represent the ICP (\ref{eq:ICP})
in a slightly different form downplaying the structure $(x_i,y_i)$ of $z_i$.
Define $\Gamma^{\epsilon}(z_1,\ldots,z_l):=\{(x,y)\mid p^y>\epsilon\}$,
where $p^y$ is defined, as before, by (\ref{eq:p}) and (\ref{eq:alphas})
(therefore, $p^y$ depends implicitly on $x$).
Proposition~\ref{prop:validity-ICP} can be restated by saying that the probability of error
$Z_{l+1}\notin\Gamma^{\epsilon}(Z_1,\ldots,Z_l)$ does not exceed $\epsilon$
provided $Z_1,\ldots,Z_{l+1}$ are exchangeable.

We consider a canonical probability space in which $Z_i=(X_i,Y_i)$,
$i=1,\ldots,l+1$,
are i.i.d.\ random examples.
A set predictor $\Gamma$
(outputting a subset of $\mathbf{Z}$ given $l$ examples and measurable in a suitable sense)
is \emph{$(\epsilon,\delta)$-valid}
if, for any probability distribution $P$ on $\mathbf{Z}$,
$$
  P^l
  \left(
    P(\Gamma(Z_1,\ldots,Z_l))\ge1-\epsilon
  \right)
  \ge
  1-\delta.
$$
It is easy to see that ICPs satisfy this property
for suitable $\epsilon$ and $\delta$.
\ifnotCONF\renewcommand{\theproposition}{\arabic{proposition}a}\fi
\begin{proposition}\label{prop:2-parameter}
  Suppose $\epsilon,\delta\in[0,1]$,
  \begin{equation}\label{eq:E}
    E
    \ge
    \epsilon + \sqrt{\frac{-\ln\delta}{2n}},
  \end{equation}
  where $n:=l-m$ is the size of the calibration set,
  and $\Gamma$ is an inductive conformal predictor.
  The set predictor $\Gamma^{\epsilon}$ is then $(E,\delta)$-valid.
  Moreover, for any probability distribution $P$ on $\mathbf{Z}$
  and any proper training set $(z_1,\ldots,z_m)\in\mathbf{Z}^m$,
  $$
    P^n
    \left(
      P(\Gamma(z_1,\ldots,z_m,Z_{m+1},\ldots,Z_l))\ge1-\epsilon
    \right)
    \ge
    1-\delta.
  $$
\end{proposition}
\ifnotCONF\addtocounter{proposition}{-1}\fi

This proposition gives the following recipe for constructing $(\epsilon,\delta)$-valid set predictors.
The recipe only works if the training set is sufficiently large;
in particular, its size $l$ should significantly exceed
$
  N := (-\ln\delta)/(2\epsilon^2).
$
Choose an ICP $\Gamma$ with the size $n$ of the calibration set
exceeding $N$.
Then the set predictor
$
  \Gamma^{\epsilon-\sqrt{(-\ln\delta)/(2n)}}
$
will be $(\epsilon,\delta)$-valid.

\ifnotCONF
  \begin{proof}[Proof of Proposition~\ref{prop:2-parameter}]
\fi
  Let $E\in(\epsilon,1)$
  (not necessarily satisfying (\ref{eq:E})).
  Fix the proper training set $(z_1,\ldots,z_m)$.
  By (\ref{eq:ICP}) and (\ref{eq:p}),
  the set predictor $\Gamma^{\epsilon}$ makes an error,
  $z_{l+1}\notin\Gamma^{\epsilon}(z_1,\ldots,z_l)$,
  if and only if the number of $i=m+1,\ldots,l$ such that $\alpha_i\le\alpha^y$
  is at most $\lfloor\epsilon(n+1)-1\rfloor$;
  in other words, if and only if $\alpha^y<\alpha_{(k)}$,
  where $\alpha_{(k)}$ is the $k$th smallest $\alpha_i$
  and $k:=\lfloor\epsilon(n+1)-1\rfloor+1$.
  Therefore, the $P$-probability of the complement of $\Gamma^{\epsilon}(z_1,\ldots,z_l)$ is
  $P(A((z_1,\ldots,z_m),Z)<\alpha_{(k)})$,
  where $A$ is the inductive conformity $m$-measure.
  Set
  $$
    \alpha^*
    :=
    \inf\{\alpha\mid P(A((z_1,\ldots,z_m),Z)<\alpha) > E\}
    \text{ and }
    \begin{cases}
      E' := P(A((z_1,\ldots,z_m),Z)<\alpha^*)\\
      E'' := P(A((z_1,\ldots,z_m),Z)\le\alpha^*).
    \end{cases}
  $$
  The $\sigma$-additivity of measures implies that $E'\le E\le E''$,
  and $E'=E=E''$ unless $\alpha^*$ is an atom of $A((z_1,\ldots,z_m),Z)$.
  Both when $E'=E$ and when $E'<E$,
  the probability of error will exceed $E$ if an only if $\alpha_{(k)}>\alpha^*$.
  In other words,
  if only if we have at most $k-1$ of the $\alpha_i$ below or equal to $\alpha^*$.
  The probability that at most $k-1=\lfloor\epsilon(n+1)-1\rfloor$ values of the $\alpha_i$
  are below or equal to $\alpha^*$
  equals $\Prob(B''_n\le\lfloor\epsilon(n+1)-1\rfloor)\le\Prob(B_n\le\lfloor\epsilon(n+1)-1\rfloor)$,
  where $B''_n\sim\bin_{n,E''}$, $B_n\sim\bin_{n,E}$,
  and $\bin_{n,p}$ stands for the binomial distribution with $n$ trials and probability of success $p$.
  \ifnotCONF(For the inequality, see Lemma~\ref{lem:binomial} below.) \fi
  By Hoeffding's inequality
  (see, e.g., \citealt{vovk/etal:2005book}, p.~287),
  the probability of error will exceed $E$ with probability at most
  \ifnotCONF\begin{multline}
    \Prob(B_n\le\lfloor\epsilon(n+1)-1\rfloor)
    \le
    \Prob(B_n\le\epsilon n)\\
    =
    \Prob(B_n/n-E\le\epsilon-E)
    \le
    \exp
    \left(
      -\frac{(\epsilon-E)^2n^2}{2n/4}
    \right)
    =
    e^{-2(E-\epsilon)^2n}.
  \end{multline}\fi
  \label{eq:chain}
  Solving
  $
    e^{-2(E-\epsilon)^2n}
    =
    \delta
  $
  we obtain that $\Gamma^{\epsilon}$ is $(E,\delta)$-valid whenever (\ref{eq:E}) is satisfied.
\ifnotCONF
  \end{proof}
\fi

\ifnotCONF
  In the proof of Proposition~\ref{prop:2-parameter} we used the following lemma.
  \begin{lemma}\label{lem:binomial}
    Fix the number of trials $n$.
    The distribution function $\bin_{n,p}(K)$ of the binomial distribution
    is decreasing in the probability of success $p$ for a fixed $K\in\{0,\ldots,n\}$.
  \end{lemma}
  \begin{proof}
    It suffices to check that
    $$
      \frac{d\bin_{n,p}(K)}{dp}
      =
      \frac{d}{dp}
      \sum_{k=0}^K\binom{n}{k}p^k(1-p)^{n-k}
      =
      \sum_{k=0}^K
      \frac{k-np}{p(1-p)}
      \binom{n}{k}p^{k}(1-p)^{n-k}
    $$
    is nonpositive for $p\in(0,1)$.
    The last sum has the same sign as the mean of the function $f(k):=k-np$
    over the set $k\in\{0,\ldots,K\}$ with respect to the binomial distribution,
    and so it remains to notice that the overall mean of $f$ is $0$ and that the function $f$ is increasing.
  \end{proof}
\fi

The inequality (\ref{eq:E}) in Proposition~\ref{prop:2-parameter} is simple
but somewhat crude as its derivation uses Hoeffding's inequality.
The following proposition is the more precise version of Proposition~\ref{prop:2-parameter}
that stops short of that last step.

\ifnotCONF\renewcommand{\theproposition}{\arabic{proposition}b}\fi
\begin{proposition}\label{prop:2-parameter-exact}
  Let $\epsilon,\delta,E\in[0,1]$.
  If $\Gamma$ is an inductive conformal predictor,
  the set predictor $\Gamma^{\epsilon}$ is $(E,\delta)$-valid
  provided
  \begin{equation}\label{eq:E-exact}
    \delta
    \ge
    \bin_{n,E}
    \left(
      \lfloor\epsilon(n+1)-1\rfloor
    \right),
  \end{equation}
  where $n:=l-m$ is the size of the calibration set
  and $\bin_{n,E}$ is the cumulative binomial distribution function
  with $n$ trials and probability of success $E$.
  If the random variable $A((z_1,\ldots,z_m),Z)$ is continuous,
  $\Gamma^{\epsilon}$ is $(E,\delta)$-valid if and only if (\ref{eq:E-exact}) holds.
\end{proposition}
\ifnotCONF\renewcommand{\theproposition}{\arabic{proposition}}\fi

\begin{proof}
  See the left-most expression in (\ref{eq:chain}) and remember that $E''=E$
  unless $\alpha^*$ is an atom of $A((z_1,\ldots,z_m),Z)$.
\end{proof}

\ifnotCONF
  \begin{remark}
    {\rm The training conditional guarantees discussed in this section
    are very similar to those for the hold-out estimate:
    compare, e.g., Proposition~\ref{prop:2-parameter-exact} above and Theorem~3.3 in \citet{langford:2005}.
    The former says that $\Gamma^{\epsilon}$ is $(E,\delta)$-valid for
    \begin{equation}\label{eq:my-inequality}
      E
      :=
      \overbin_{n,\delta}
      \left(
        \lfloor\epsilon(n+1)-1\rfloor
      \right)
      \le
      \overbin_{n,\delta}
      \left(
        \epsilon n
      \right)
    \end{equation}
    where $\overbin$ is the inverse function to $\bin$:
    \begin{equation*}
      \overbin_{n,\delta}(k)
      :=
      \max\{p\mid \bin_{n,p}(k)\ge\delta\}
    \end{equation*}
    (unless $k=n$, we can also say that $\overbin_{n,\delta}(k)$ is the only value of $p$
    such that $\bin_{n,p}(k)=\delta$:
    cf.\ Lemma~\ref{lem:binomial} above).
    And the latter says that a point predictor's error probability (over the test example) does not exceed
    \begin{equation}\label{eq:hold-out}
      \overbin_{n,\delta}
      \left(
        k
      \right)
    \end{equation}
    with probability at least $1-\delta$ (over the training set),
    where $k$ is the number of errors on a held-out set of size $n$.
    The main difference between (\ref{eq:my-inequality}) and (\ref{eq:hold-out})
    is that whereas one inequality contains the approximate expected number of errors $\epsilon n$
    for $n$ new examples
    the other contains the actual number of errors $k$ on $n$ examples.
    Several researchers have found that the hold-out estimate is surprisingly difficult to beat;
    however, like the ICP of this section, it is not example conditional at all.}
  \end{remark}

  \begin{remark}
    {\rm Inequality (\ref{eq:E-exact}) can be rewritten as
    \begin{equation*}
      E
      \ge
      \overbin_{n,\delta}
      \left(
        \lfloor\epsilon(n+1)-1\rfloor
      \right).
    \end{equation*}
    In combination with inequality 2.\ in \citet{langford:2005}, p.~278,
    this shows that Proposition~\ref{prop:2-parameter} will continue to hold
    if (\ref{eq:E}) is replaced by
    \begin{equation*}
      E
      \ge
      \epsilon
      +
      \sqrt{\frac{-2\epsilon\ln\delta}{n}}
      -
      \frac{2\ln\delta}{n}.
    \end{equation*}
    The last inequality is weaker than (\ref{eq:E}) for small $\epsilon$.}
  \end{remark}
\fi

\section{Conditional inductive conformal predictors}
\label{sec:CICP}

The motivation behind conditional inductive conformal predictors
is that ICPs do not always achieve the required probability $\epsilon$
of error $Y_{l+1}\notin\Gamma^{\epsilon}(Z_1,\ldots,Z_l,X_{l+1})$
conditional on $(X_{l+1},Y_{l+1})\in E$ for important sets $E\subseteq\mathbf{Z}$.
This is often undesirable.
If, e.g., our set predictor is valid at the significance level $5\%$
but makes an error with probability $10\%$ for men and $0\%$ for women,
both men and women can be unhappy with calling $5\%$ the probability of error.
Moreover, in many problems we might want different significance levels
for different regions of the example space:
e.g., in the problem of spam detection (considered in Section~\ref{sec:experiments})
classifying spam as email usually makes much less harm than classifying email as spam.

An \emph{inductive $m$-taxonomy} is a measurable function
$K:\mathbf{Z}^m\times\mathbf{Z}\to\mathbf{K}$,
where $\mathbf{K}$ is a measurable space.
Usually the \emph{category} $K((z_1,\ldots,z_m),z)$ of an example $z$
is a kind of classification of $z$,
which may depend on the proper training set $(z_1,\ldots,z_m)$.

The \emph{conditional inductive conformal predictor} (conditional ICP)
corresponding to $K$ and an inductive conformity $m$-measure $A$
is defined as the set predictor (\ref{eq:ICP}),
where the p-values $p^y$ are now defined by
\begin{equation}\label{eq:p-cond}
  p^y
  :=
  \frac
  {
    \left|\left\{
      i=m+1,\ldots,l \mid \kappa_i=\kappa^y \;\&\; \alpha_i\le\alpha^y
    \right\}\right|
    +
    1
  }
  {
    \left|\left\{
      i=m+1,\ldots,l \mid \kappa_i=\kappa^y
    \right\}\right|
    +
    1
  },
\end{equation}
the categories $\kappa$ are defined by
\begin{equation*}
  \kappa_i
  :=
  K((z_1,\ldots,z_m),z_i),
  \quad
  i=m+1,\ldots,l,
  \qquad
  \kappa^y
  :=
  K((z_1,\ldots,z_m),(x,y)),
\end{equation*}
and the conformity scores $\alpha$ are defined as before by (\ref{eq:alphas}).
A \emph{label conditional ICP} is a  conditional ICP with the inductive $m$-taxonomy $K(\cdot,(x,y)):=y$.

The following proposition is the conditional analogue of Proposition~\ref{prop:validity-ICP};
in particular, it shows that in classification problems label conditional ICPs
achieve label conditional validity.

\begin{proposition}\label{prop:validity-ICP-cond}
  If random examples $Z_{m+1},\ldots,Z_l,Z_{l+1}=(X_{l+1},Y_{l+1})$ are exchangeable,
  the probability of error $Y_{l+1}\notin\Gamma^{\epsilon}(Z_1,\ldots,Z_l,X_{l+1})$
  given the category $K((Z_1,\ldots,Z_m),Z_{l+1})$ of $Z_{l+1}$
  does not exceed $\epsilon$ for any $\epsilon$ and any conditional inductive conformal predictor $\Gamma$
  corresponding to $K$.
\end{proposition}

\section{Object conditional validity}
\label{sec:O}


In this section we prove a negative result
(a version of Lemma~1 in \citealt{lei/wasserman:2012})
which says that the requirement of precise object conditional validity cannot be satisfied
in a non-trivial way for rich object spaces (such as $\mathbb{R}$).
If $P$ is a probability distribution on $\mathbf{Z}$,
we let $P_{\mathbf{X}}$ stand for its marginal distribution on $\mathbf{X}$:
$P_{\mathbf{X}}(A):=P(A\times\mathbf{Y})$.
Let us say that a set predictor $\Gamma$ \emph{has $1-\epsilon$ object conditional validity},
where $\epsilon\in(0,1)$, if,
for all probability distributions $P$ on $\mathbf{Z}$ and $P_{\mathbf{X}}$-almost all $x\in\mathbf{X}$,
\begin{equation}\label{eq:conditional-validity}
  P^{l+1}
  \left(
    Y_{l+1}\in\Gamma(Z_1,\ldots,Z_l,X_{l+1})
    \mid X_{l+1}=x
  \right)
  \ge
  1-\epsilon.
\end{equation}
The Lebesgue measure on $\mathbb{R}$ will be denoted $\Lambda$.
If $Q$ is a probability distribution,
we say that a property $F$ holds for \emph{$Q$-almost all} elements of a set $E$
if $Q(E\setminus F)=0$;
a \emph{$Q$-non-atom} is an element $x$ such that $Q(\{x\})=0$.

\begin{proposition}\label{prop:negative}
  Suppose $\mathbf{X}$ is a separable metric space equipped with the Borel $\sigma$-algebra.
  Let $\epsilon\in(0,1)$.
  Suppose that a set predictor $\Gamma$ has $1-\epsilon$ object conditional validity.
  In the case of regression, we have,
  for all $P$ and for $P_{\mathbf{X}}$-almost all $P_{\mathbf{X}}$-non-atoms $x\in\mathbf{X}$,
  \begin{equation}\label{eq:negative-regression}
    P^l
    \left(
      \Lambda(\Gamma(Z_1,\ldots,Z_l,x))
      =
      \infty
    \right)
    \ge
    1-\epsilon.
  \end{equation}
  In the case of classification, we have,
  for all $P$, all $y\in\mathbf{Y}$, and $P_{\mathbf{X}}$-almost all $P_{\mathbf{X}}$-non-atoms~$x$,
  \begin{equation}\label{eq:negative-classification}
    P^l
    \left(
      y\in\Gamma(Z_1,\ldots,Z_l,x)
    \right)
    \ge
    1-\epsilon.
  \end{equation}
\end{proposition}

We are mainly interested in the case of a small $\epsilon$ (corresponding to high confidence),
and in this case (\ref{eq:negative-regression}) implies that, in the case of regression, prediction intervals
(i.e., the convex hulls of prediction sets)
can be expected to be infinitely long
unless the new object is an atom.
In the case of classification,
(\ref{eq:negative-classification}) says that each particular $y\in\mathbf{Y}$
is likely to be included in the prediction set,
and so the prediction set is likely to be large.
In particular, 
(\ref{eq:negative-classification}) implies that the expected size of the prediction set
is a least $(1-\epsilon)\left|\mathbf{Y}\right|$.

Of course, the condition that $x$ be a non-atom is essential:
if $P_{\mathbf{X}}(\{x\})>0$,
an inductive conformal predictor that ignores all examples with objects different from $x$
will have $1-\epsilon$ object conditional validity
and can give narrow predictions if the training set is big enough
to contain many examples with $x$ as their object.

\ifnotCONF
\begin{remark}
  {\rm Nontrivial set predictors having $1-\epsilon$ object conditional validity
  are constructed by \citet{mccullagh/etal:2009-full}
  assuming the Gauss linear model.}
\end{remark}
\fi

\ifnotCONF
  \begin{proof}[Proof of Proposition~\ref{prop:negative}]
\fi
  The proof will be based on the ideas of Lei and Wasserman
  (\citeyear{lei/wasserman:2012}, the proof of Lemma~1).

  Suppose (\ref{eq:negative-regression}) does not hold on a measurable set $E$
  of $P_{\mathbf{X}}$-non-atoms $x\in\mathbf{X}$ such that $P_{\mathbf{X}}(E)>0$.
  Shrink $E$ in such a way that $P_{\mathbf{X}}(E)>0$ still holds
  but there exists $\delta>0$ and $C>0$ such that,
  for each $x\in E$,
  \begin{equation}\label{eq:shrink}
    P^l
    \left(
      \Lambda(\Gamma(Z_1,\ldots,Z_l,x))
      \le
      C
    \right)
    \ge
    \epsilon+\delta.
  \end{equation}
  Let $V$ be the total variation distance between probability measures,
  $V(P,Q):=\sup_A\left|P(A)-Q(A)\right|$;
  we then have
  $$
    V(P^l,Q^l)
    \le
    \sqrt{2}\sqrt{1-(1-V(P,Q))^l}
  $$
  (this follows from the connection of $V$ with the Hellinger distance:
  see, e.g., \citealt{tsybakov:2010}, Section 2.4).
  Shrink $E$ further so that $P_{\mathbf{X}}(E)>0$ still holds but
  \begin{equation}\label{eq:small-total-variation}
    \sqrt{2}\sqrt{1-(1-P_{\mathbf{X}}(E))^l}
    \le
    \delta/2.
  \end{equation}
  (This can be done under our assumption that $\mathbf{X}$ is a separable metric space:
  \ifnotCONF
    see Lemma~\ref{lem:shrink} below\fi.)
  Define another probability distribution $Q$ on $\mathbf{Z}$ by the requirements that
  $Q(A\times B)=P(A\times B)$ for all measurable $A\subseteq(\mathbf{X}\setminus E)$, $B\subseteq\mathbb{R}$
  and $Q(A\times B)=P_{\mathbf{X}}(A)\times U(B)$ for all measurable $A\subseteq E$, $B\subseteq\mathbb{R}$,
  where $U$ is the uniform probability distribution on the interval $[-DC,DC]$
  and $D>0$ will be chosen below.
  Since $V(P,Q)\le P_{\mathbf{X}}(E)$,
  we have $V(P^l,Q^l)\le\delta/2$;
  therefore, by (\ref{eq:shrink}),
  \begin{equation*}
    Q^l
    \left(
      \Lambda(\Gamma(Z_1,\ldots,Z_l,x))
      \le
      C
    \right)
    \ge
    \epsilon+\delta/2
  \end{equation*}
  for each $x\in E$.
  The last inequality implies,
  by Fubini's theorem,
  \begin{equation*}
    Q^{l+1}
    \left(
      \Lambda(\Gamma(Z_1,\ldots,Z_l,X_{l+1}))
      \le
      C
      \;\&\;
      X_{l+1}\in E
    \right)
    \ge
    \left(
      \epsilon+\delta/2
    \right)
    Q_{\mathbf{X}}(E),
  \end{equation*}
  where $Q_{\mathbf{X}}(E)=P_{\mathbf{X}}(E)>0$ is the marginal $Q$-probability of $E$.
  When $D=D(\delta Q_{\mathbf{X}}(E),C)$ is sufficiently large this in turn implies
  \begin{equation*}
    Q^{l+1}
    \left(
      Y_{l+1}\notin\Gamma(Z_1,\ldots,Z_l,X_{l+1})
      \;\&\;
      X_{l+1}\in E
    \right)
    \ge
    \left(
      \epsilon+\delta/4
    \right)
    Q_{\mathbf{X}}(E).
  \end{equation*}
  However, the last inequality contradicts
  \begin{equation*}
    \frac
    {
      Q^{l+1}
      \left(
        Y_{l+1}\notin\Gamma(Z_1,\ldots,Z_l,X_{l+1})
        \;\&\;
        X_{l+1}\in E
      \right)
    }
    {Q_{\mathbf{X}}(E)}
    \le
    \epsilon,
  \end{equation*}
  which follows from $\Gamma$ having $1-\epsilon$ object conditional validity
  and the definition of conditional probability.

  It remains to consider the case of classification.
  Suppose (\ref{eq:negative-classification}) does not hold on a measurable set $E$
  of $P_{\mathbf{X}}$-non-atoms $x\in\mathbf{X}$ such that $P_{\mathbf{X}}(E)>0$.
  Shrink $E$ in such a way that $P_{\mathbf{X}}(E)>0$ still holds
  but there exists $\delta>0$ such that,
  for each $x\in E$,
  \begin{equation*}
    P^l
    \left(
      y\in\Gamma(Z_1,\ldots,Z_l,x)
    \right)
    \le
    1-\epsilon-\delta.
  \end{equation*}
  Without loss of generality we further assume that (\ref{eq:small-total-variation}) also holds.
  Define a probability distribution $Q$ on $\mathbf{Z}$ by the requirements that
  $Q(A\times B)=P(A\times B)$ for all measurable $A\subseteq(\mathbf{X}\setminus E)$ and all $B\subseteq\mathbf{Y}$
  and that $Q(A\times\{y\})=P_{\mathbf{X}}(A)$ for all measurable $A\subseteq E$
  (i.e., modify $P$ setting the conditional distribution of $Y$ given $X\in E$ to the unit mass
  concentrated at $y$).
  Then for each $x\in E$ we have
  \begin{equation*}
    Q^l
    \left(
      y\in\Gamma(Z_1,\ldots,Z_l,x)
    \right)
    \le
    1-\epsilon-\delta/2,
  \end{equation*}
  which implies
  \begin{equation*}
    Q^{l+1}
    \left(
      Y_{l+1}\in\Gamma(Z_1,\ldots,Z_l,X_{l+1})
      \;\&\;
      X_{l+1}\in E
    \right)
    \le
    \left(
      1-\epsilon-\delta/2
    \right)
    Q_{\mathbf{X}}(E).
  \end{equation*}
  The last inequality contradicts $\Gamma$ having $1-\epsilon$ object conditional validity.
\ifnotCONF
  \end{proof}
\fi

\ifnotCONF
  In the proof of Proposition \ref{prop:negative} we used the following lemma.
  \begin{lemma}\label{lem:shrink}
    If $Q$ is a probability measure on $\mathbf{X}$, which a separable metric space,
    $E$ is a set of $Q$-non-atoms such that $Q(E)>0$, and $\delta>0$ is an arbitrarily small number,
    then there is $E'\subseteq E$ such that $Q(E')<\delta$.
  \end{lemma}
  \begin{proof}
    We can take the intersection of $E$ and an open ball centered at any element of $\mathbf{X}$
    for which all such intersections have a positive $Q$-probability.
    Let us prove that such elements exist.
    Suppose they do not.

    Fix a countable dense subset $A_1$ of $\mathbf{X}$.
    Let $A_2$ be the union of all open balls $B$ with rational radii
    centered at points in $A_1$ such that $Q(B\cap E)=0$.
    On one hand, the $\sigma$-additivity of measures implies $Q(A_2\cap E)=0$.
    On the other hand, $A_2=\mathbf{X}$:
    indeed, for each $x\in\mathbf{X}$
    there is an open ball $B$ of some radius $\delta>0$ centered at $x$ that satisfies $Q(B\cap E)=0$;
    since $x$ belongs to the radius $\delta/2$ open ball
    centered at a point in $A_1$ at a distance of less than $\delta/2$ from $x$,
    we have $x\in A_2$.
    This contradicts $Q(E)>0$.
  \end{proof}
\fi

Proposition~\ref{prop:negative} can be extended to randomized set predictors $\Gamma$
(in which case $P^l$ and $P^{l+1}$ in expressions such as (\ref{eq:conditional-validity})
and (\ref{eq:negative-regression}) should be replaced by the probability distribution
comprising both $P$ and the internal coin tossing of $\Gamma$).
This clarifies the provenance of $\epsilon$
in (\ref{eq:negative-regression}) and (\ref{eq:negative-classification}):
$\epsilon$ cannot be replaced by a smaller constant
since the set predictor predicting $\mathbf{Y}$ with probability $1-\epsilon$
and $\emptyset$ with probability $\epsilon$ has $1-\epsilon$ object conditional validity.

Proposition~\ref{prop:negative} does not prevent the existence of efficient set predictors
that are conditionally valid in an asymptotic sense;
indeed, the paper by \citet{lei/wasserman:2012} is devoted
to constructing asymptotically efficient and asymptotically conditionally valid set predictors
in the case of regression.

\section{Experiments}
\label{sec:experiments}

This section describes some simple experiments on the well-known \texttt{Spambase} data set
contributed by George Forman to the UCI Machine Learning Repository
\citep{UCI:data}.
Its overall size is 4601 examples and it contains examples of two classes:
\texttt{email} (also written as 0) and \texttt{spam} (also written as 1).
\citet{hastie/etal:2009} report results of several machine-learning algorithms on this data set
split randomly into a training set of size 3065 and test set of size 1536.
The best result is achieved by MART (multiple additive regression tree;
$4.5\%$ error rate according to the second edition of \citealt{hastie/etal:2009}).

We randomly permute the data set and divide it into 2602 examples for the proper training set,
999 for the calibration set, and 1000 for the test set.
\ifnotCONF
  Our split between the proper training, calibration, and test sets, approximately 4:1:1,
  is inspired by the standard recommendation for the allocation of data
  into training, validation, and test sets
  (see, e.g., \citealt{hastie/etal:2009}, Section~7.2).
\fi
We consider the ICP whose conformity measure is defined by (\ref{eq:score})
where $f$ is output by MART and
\begin{equation}\label{eq:Delta}
  \Delta(y,f(x))
  :=
  \begin{cases}
    f(x) & \text{if $y=1$}\\
    -f(x) & \text{if $y=0$}.
  \end{cases}
\end{equation}
MART's output $f(x)$ models the log-odds of \texttt{spam} vs \texttt{email},
\begin{equation*} 
  f(x)
  =
  \log\frac
  {P(1\mid x)}
  {P(0\mid x)},
\end{equation*}
which makes the interpretation of (\ref{eq:Delta}) as conformity score very natural.

\ifnotCONF
  The R programs used in the experiments described in this section are available
  from the web site \url{http://alrw.net};
  the programs use the \texttt{gbm} package
  with virtually all parameters set to the default values
  (given in the description provided in response to \verb'help("gbm")').
\fi

The upper left plot in Figure~\ref{fig:scatter}
is the scatter plot of the pairs $(p^{{\rm email}},p^{{\rm spam}})$
produced by the ICP for all examples in the test set.
Email is shown as green noughts and spam as red crosses
(and it is noticeable that the noughts were drawn after the crosses).
The other two plots in the upper row are for email and spam separately.
Ideally, email should be close to the horizontal axis and spam to the vertical axis;
we can see that this is often true, with a few exceptions.
The picture for the label conditional ICP looks almost identical:
see the lower row of Figure~\ref{fig:scatter}.
\ifnotCONF
  However, on the log scale the difference becomes more noticeable:
  see Figure~\ref{fig:scatter_log}.
\fi

\begin{figure}[tb]
\begin{center}
  \includegraphics[width=\columnwidth]{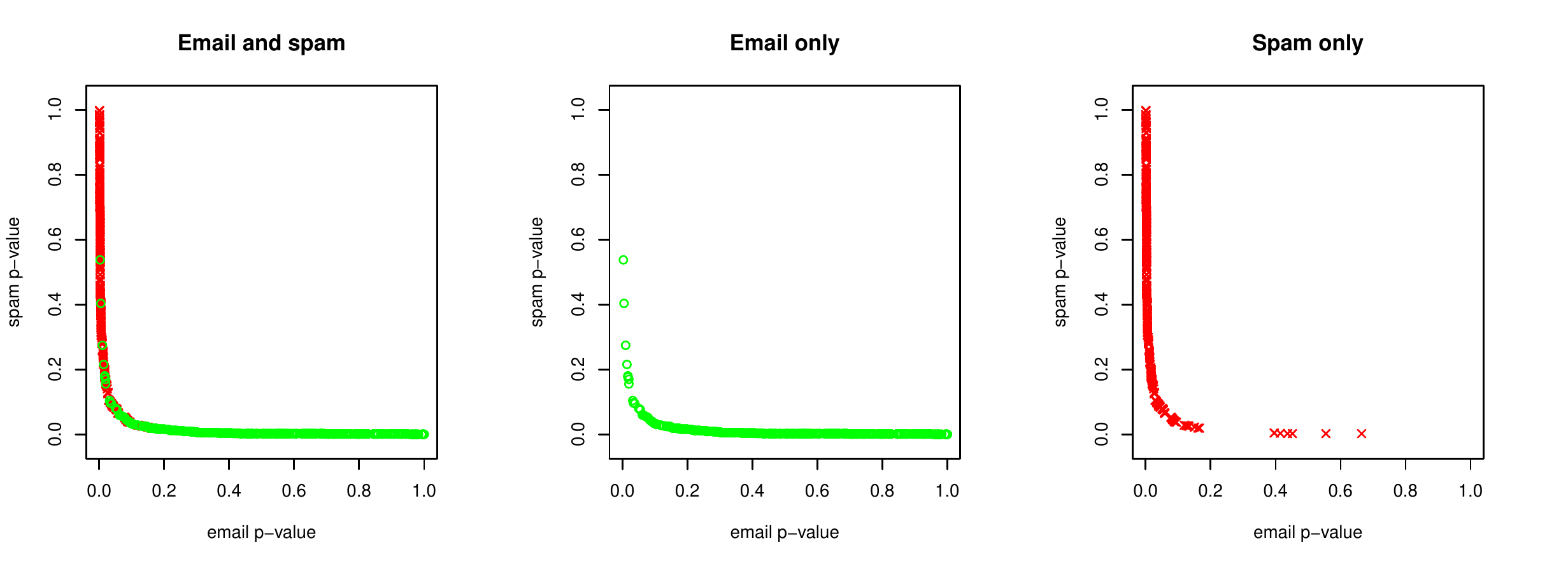}\\
  \includegraphics[width=\columnwidth]{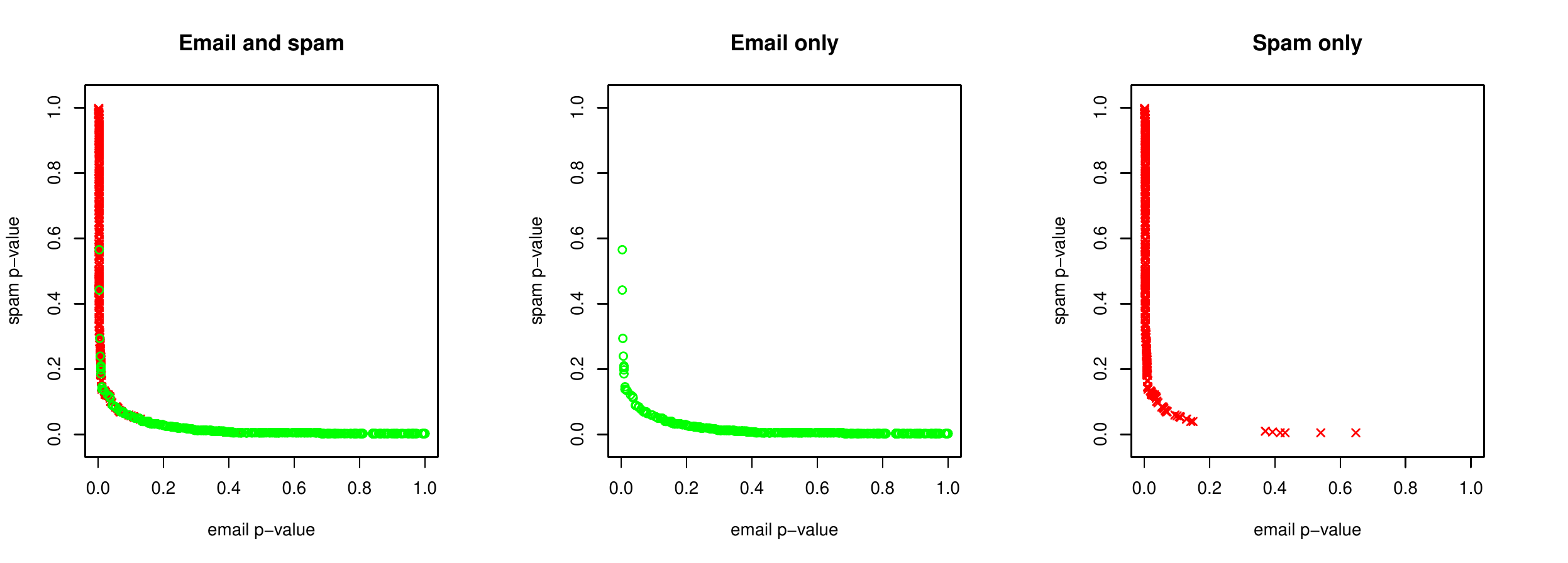}
\end{center}

\vspace{-7mm}

\caption{Scatter plots of the pairs $(p^{{\rm email}},p^{{\rm spam}})$
for all examples in the test set (left plots),
for email only (middle), and for spam only (right).
The three upper plots are for the ICP and the three lower ones are for the label conditional ICP.}
\label{fig:scatter}
\end{figure}

\ifnotCONF
  \begin{figure}[tb]
  \begin{center}
    \includegraphics[width=\columnwidth]{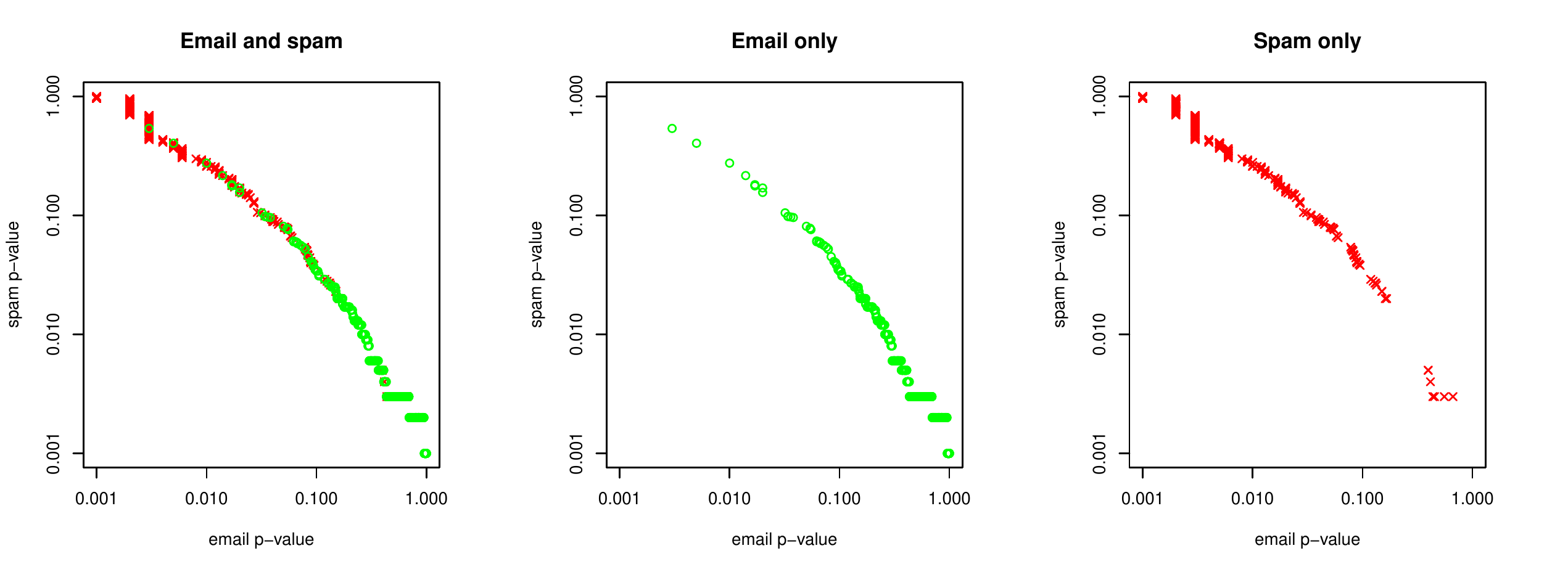}\\%
    \includegraphics[width=\columnwidth]{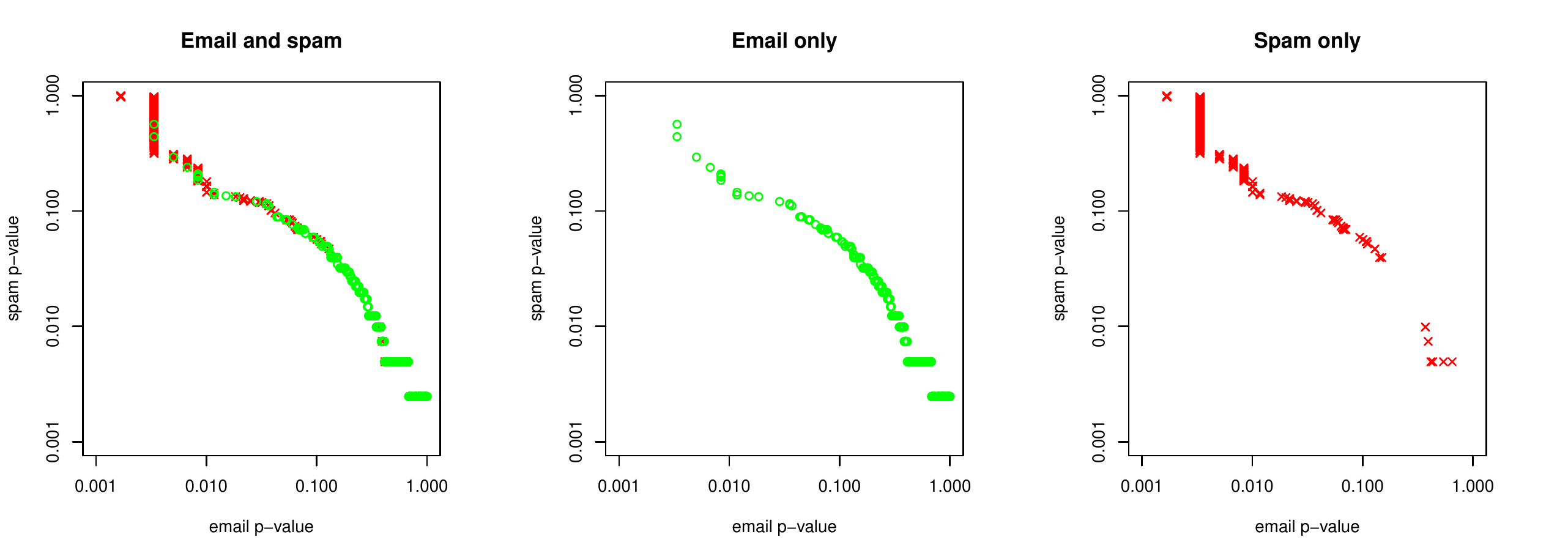}
  \end{center}

  \vspace{-7mm}

  \caption{The analogue of Figure~\ref{fig:scatter} on the log scale.}
  \label{fig:scatter_log}
  \end{figure}
\fi

Table~\ref{tab:statistics-ICP} gives some statistics for the numbers of errors,
multiple, and empty set predictions in the case of the (unconditional) ICP $\Gamma^{5\%}$
at significance level $5\%$
(we obtain different numbers not only because of different splits but also because MART is randomized;
the columns of the table correspond to the pseudorandom number generator seeds 0, 1, 2, etc.).
The table demonstrates the validity, (lack of) conditional validity, and efficiency of the algorithm
(the latter is of course inherited from the efficiency of MART).
We give two kinds of conditional figures:
the percentages of errors, multiple, and empty predictions for different labels
and for two different kinds of objects.
The two kinds of objects are obtained by splitting the object space $\mathbf{X}$
by the value of an attribute that we denote $\$$:
it shows the percentage of the character $\$$ in the text of the message.
The condition $\$<5.55\%$ was the root of the decision tree chosen
both by Hastie et al.\ (\citeyear{hastie/etal:2009}, Section~9.2.5),
who use all attributes in their analysis,
and by Maindonald and Braun (\citeyear{maindonald/braun:2007}, Chapter~11),
who use 6 attributes chosen by them manually.
(Both books use the \texttt{rpart} R package for decision trees.)

Notice that the numbers of errors, multiple predictions, and empty predictions
tend to be greater for spam than for email.
Somewhat counter-intuitively, they also tend to be greater
for ``email-like'' objects containing few $\$$ characters
than for ``spam-like'' objects.
The percentage of multiple and empty predictions is relatively small
since the error rate of the underlying predictor happens to be close
to our significance level of $5\%$.

In practice, using a fixed significance level (such as the standard $5\%$) is not a good idea;
we should at least pay attention to what happens at several significance levels.
However, experimenting with prediction sets at a fixed significance level
facilitates a comparison with theoretical results.

\begin{table}[tb]
\hspace*{-2mm}{\small\begin{tabular}{lrrrrrrrrr}
  RNG seed        & 0 & 1 & 2 & 3 & 4 & 5 & 6 & 7 & Average\\
  \hline
  errors overall  & 4.1\% & 6.9\% & 4.6\% & 5.4\% & 5.3\% & 6.1\% & 7.7\% & 5.9\% & 5.75\%\\
  \quad for email & 2.44\% & 4.61\% & 2.26\% & 3.10\% & 4.49\% & 3.98\% & 5.02\% & 3.22\% & 3.64\%\\
  \quad for spam  & 6.77\% & 10.43\% & 8.42\% & 9.02\% & 6.53\% & 9.32\% & 11.69\% & 10.29\% & 9.06\%\\
  \quad for $\$<5.55\%$  & 4.36\% & 7.91\% & 5.15\% & 6.21\% & 6.27\% & 7.89\% & 8.79\% & 7.04\% & 6.70\%\\
  \quad for $\$>5.55\%$  & 3.29\% & 4.12\% & 2.69\% & 2.64\% & 2.40\% & 1.13\% & 4.42\% & 2.15\% & 2.86\%\\
  \hline
  multiple overall & 2.7\% & 0\% & 0.1\% & 0\% & 0\% & 0.5\% & 0\% & 0\% & 0.41\%\\
  \quad for email  & 2.11\% & 0\% & 0.16\% & 0\% & 0\% & 0.33\% & 0\% & 0\% & 0.33\%\\
  \quad for spam   & 3.65\% & 0\% & 0\% & 0\% & 0\% & 0.76\% & 0\% & 0\% & 0.55\%\\
  \quad for $\$<5.55\%$  & 3.04\% & 0\% & 0.13\% & 0\% & 0\% & 0.68\% & 0\% & 0\% & 0.48\%\\
  \quad for $\$>5.55\%$  & 1.65\% & 0\% & 0\% & 0\% & 0\% & 0\% & 0\% & 0\% & 0.21\%\\
  \hline
  empty overall    & 0\% & 2.7\% & 0\% & 1.2\% & 0.8\% & 0\% & 2.5\% & 0.4\% & 0.95\%\\
  \quad for email  & 0\% & 1.48\% & 0\% & 0.65\% & 0.83\% & 0\% & 1.51\% & 0.64\% & 0.64\%\\
  \quad for spam   & 0\% & 4.58\% & 0\% & 2.06\% & 0.75\% & 0\% & 3.98\% & 0\% & 1.42\%\\
  \quad for $\$<5.55\%$  & 0\% & 3.14\% & 0\% & 1.55\% & 0.80\% & 0\% & 3.06\% & 0.52\% & 1.13\%\\
  \quad for $\$>5.55\%$  & 0\% & 1.50\% & 0\% & 0\% & 0.80\% & 0\% & 0.80\% & 0\% & 0.39\%\\
  \hline
\end{tabular}}
\caption{Percentage of errors, multiple predictions, and empty predictions on the full test set
  and separately on email and spam.
  The results are given for various values of the seed for the R (pseudo)random number generator (RNG);
  column ``Average'' gives the average values for all 8 seeds 0--7.}
\label{tab:statistics-ICP}
\end{table}

Table~\ref{tab:statistics-LCICP} gives similar statistics in the case of the label conditional ICP.
The error rates are now about equal for email and spam, as expected.
We refrain from giving similar predictable results for ``object conditional'' ICP
with $\$<5.55\%$ and $\$>5.55\%$ as categories.

\begin{table}[tb]
\hspace*{-2mm}{\small\begin{tabular}{lrrrrrrrrr}
  RNG seed        & 0 & 1 & 2 & 3 & 4 & 5 & 6 & 7 & Average\\
  \hline
  errors overall  & 3.4\% & 6.0\% & 3.8\% & 4.8\% & 5.7\% & 5.3\% & 6.5\% & 5.4\% & 5.11\%\\
  \quad for email & 3.73\% & 6.92\% & 3.87\% & 4.90\% & 6.64\% & 4.98\% & 5.85\% & 3.86\% & 5.10\%\\
  \quad for spam  & 2.86\% & 4.58\% & 3.68\% & 4.64\% & 4.27\% & 5.79\% & 7.46\% & 7.92\% & 5.15\%\\
  \hline
  multiple overall & 4.2\% & 0\% & 4.0\% & 0\% & 0\% & 0.5\% & 0\% & 0.5\% & 1.15\%\\
  \quad for email  & 3.90\% & 0\% & 5.48\% & 0\% & 0\% & 0.66\% & 0\% & 0.48\% & 1.32\%\\
  \quad for spam   & 4.69\% & 0\% & 1.58\% & 0\% & 0\% & 0.25\% & 0\% & 0.53\% & 0.88\%\\
  \hline
  empty overall    & 0\% & 1.0\% & 0\% & 0\% & 0.6\% & 0\% & 1.0\% & 0\% & 0.33\%\\
  \quad for email  & 0\% & 1.48\% & 0\% & 0\% & 0.83\% & 0\% & 0.67\% & 0\% & 0.37\%\\
  \quad for spam   & 0\% & 0.25\% & 0\% & 0\% & 0.25\% & 0\% & 1.49\% & 0\% & 0.25\%\\
  \hline
\end{tabular}}
\caption{The analogue of a subset of Table~\ref{tab:statistics-ICP} in the case of the label conditional ICP.}
\label{tab:statistics-LCICP}
\end{table}

Figure~\ref{fig:calibration_ICP} gives the calibration plots of the ICP for the test set.
It shows approximate validity even for email and spam separately,
except for the all-important lower-left corners.
The latter are shown separately in Figure~\ref{fig:calibration_ICP_LL},
where the lack of conditional validity becomes evident;
cf.\ Figure~\ref{fig:calibration_LCICP_LL} for the label conditional ICP.

\begin{figure}[tb]
\begin{center}
  \includegraphics[width=\columnwidth]{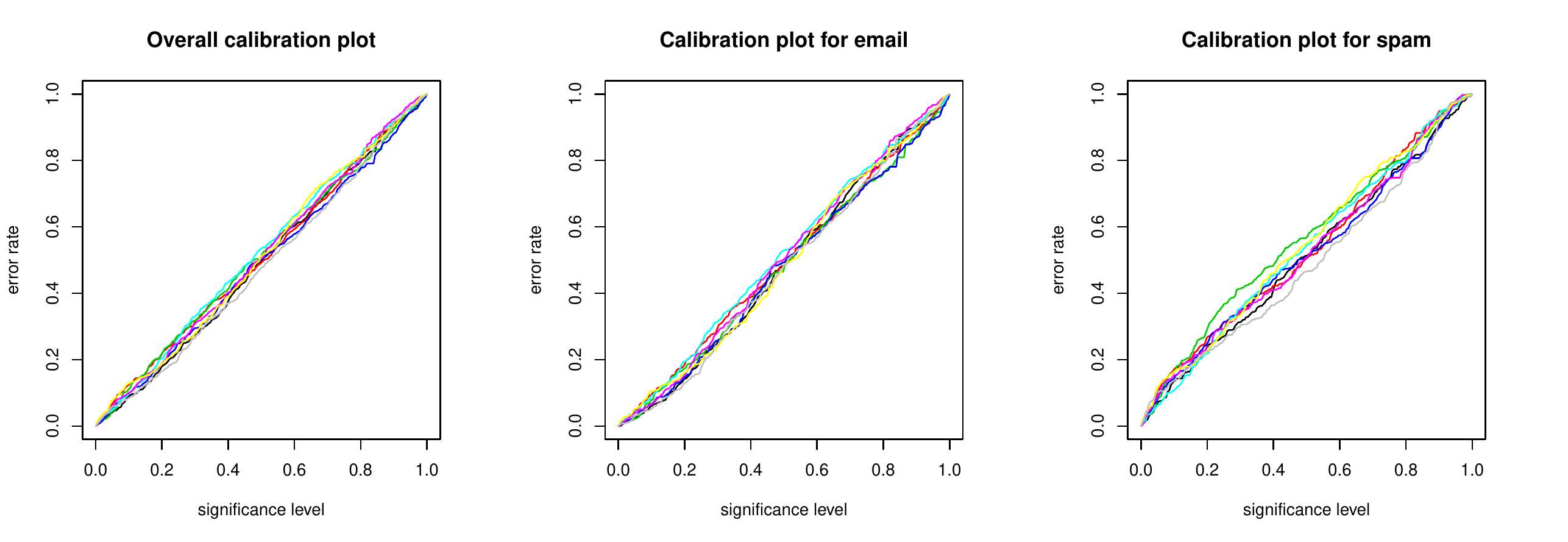}
\end{center}

\vspace{-7mm}

\caption{The calibration plot for the test set overall, the email in the test set, and the spam in the test set
  (for the first 8 seeds, 0--7).}
\label{fig:calibration_ICP}
\end{figure}

\begin{figure}[tb]
\begin{center}
  \includegraphics[width=\columnwidth]{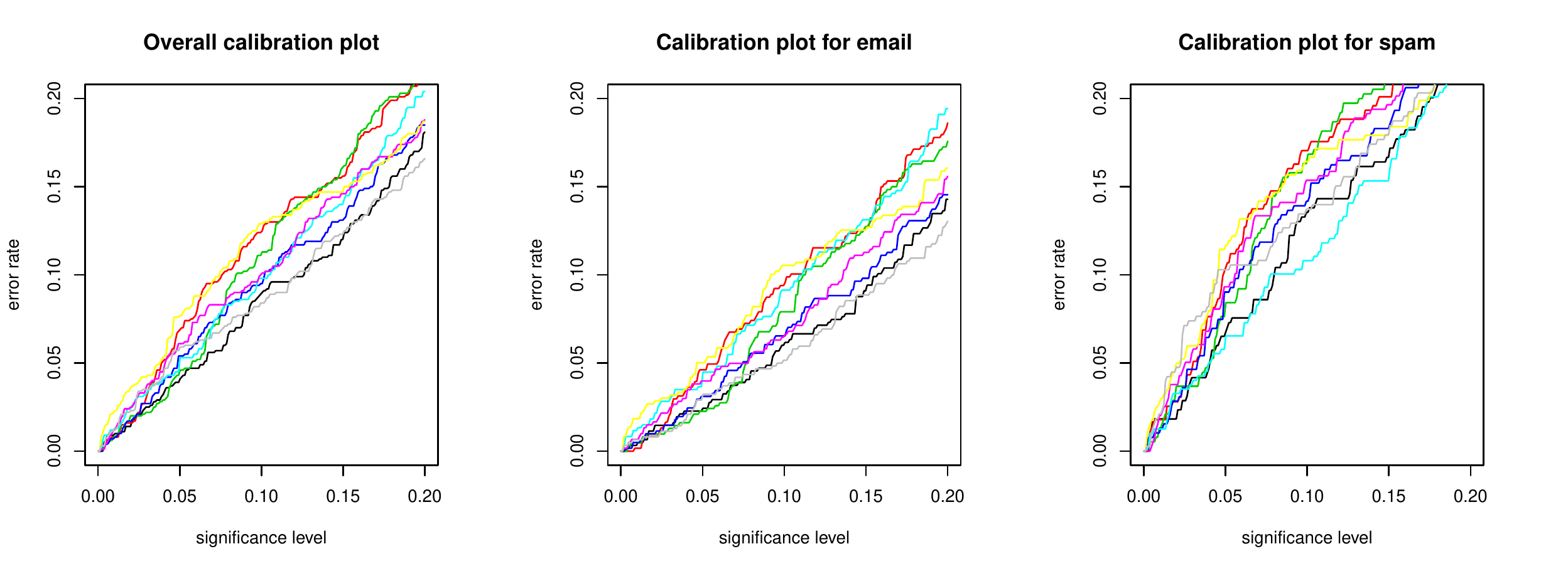}
\end{center}

\vspace{-7mm}

\caption{The lower left corners of the plots in Figure~\ref{fig:calibration_ICP}.}
\label{fig:calibration_ICP_LL}
\end{figure}

\begin{figure}[tb]
\begin{center}
  \includegraphics[width=\columnwidth]{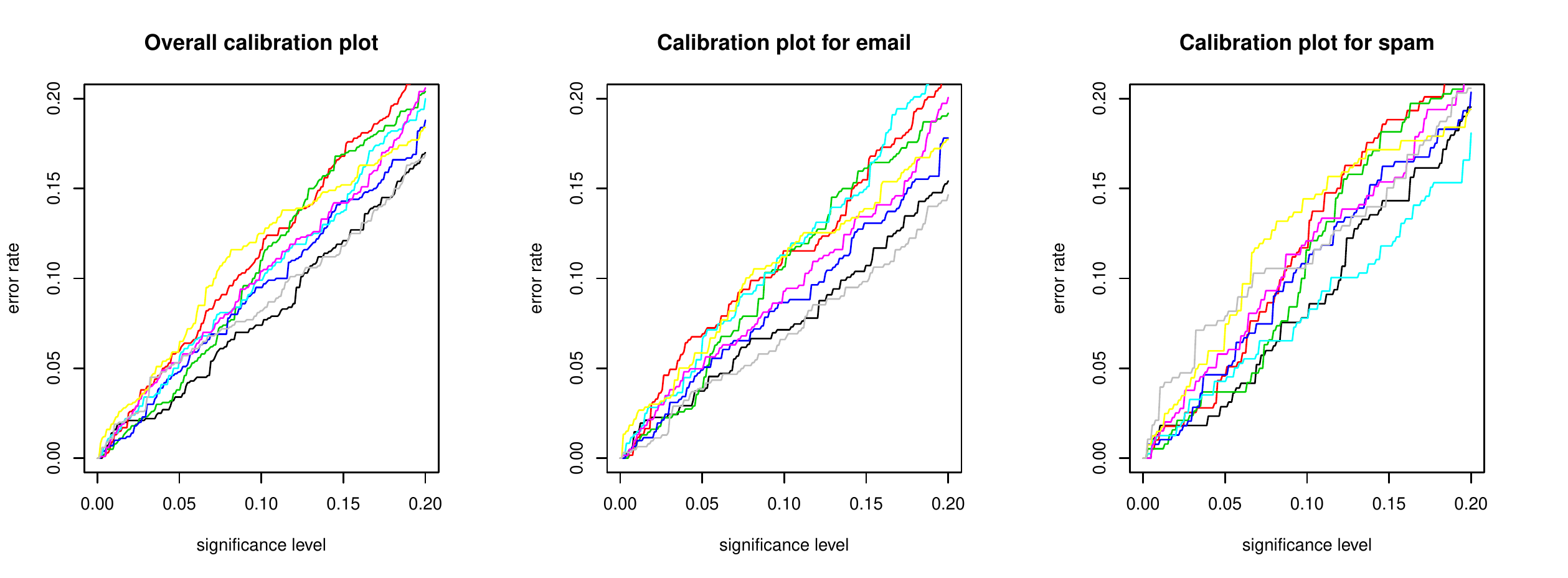}
\end{center}

\vspace{-7mm}

\caption{The analogue of Figure~\ref{fig:calibration_ICP_LL} for the label conditional ICP.}
\label{fig:calibration_LCICP_LL}
\end{figure}

From the numbers given in the ``errors overall'' row of Table~\ref{tab:statistics-ICP}
we can extract the corresponding confidence intervals for the probability of error
conditional on the training set and MART's internal coin tosses;
these are shown in Figure~\ref{fig:train_cond_exper}.
It can be seen that training conditional validity is not grossly violated.
(Notice that the 8 training sets used for producing this figure are not completely independent.
Besides, the assumption of randomness might not be completely satisfied:
permuting the data set ensures exchangeability but not necessarily randomness.)
It is instructive to compare Figure~\ref{fig:train_cond_exper}
with the ``theoretical'' Figure~\ref{fig:train_cond_theor}
obtained from Propositions~\ref{prop:2-parameter-exact} (the thick blue line)
and ~\ref{prop:2-parameter} (the thin red line).
The dotted green line corresponds to the significance level $5\%$,
and the black dot roughly corresponds to the maximal expected probability of error
among 8 randomly chosen training sets.
(It might appear that there is a discrepancy between Figures~\ref{fig:train_cond_exper}
and \ref{fig:train_cond_theor},
but choosing different seeds usually leads to smaller numbers of errors
than in Figure~\ref{fig:train_cond_exper}.)

\begin{figure}[tb]
\begin{center}
  \includegraphics[width=0.4\columnwidth]{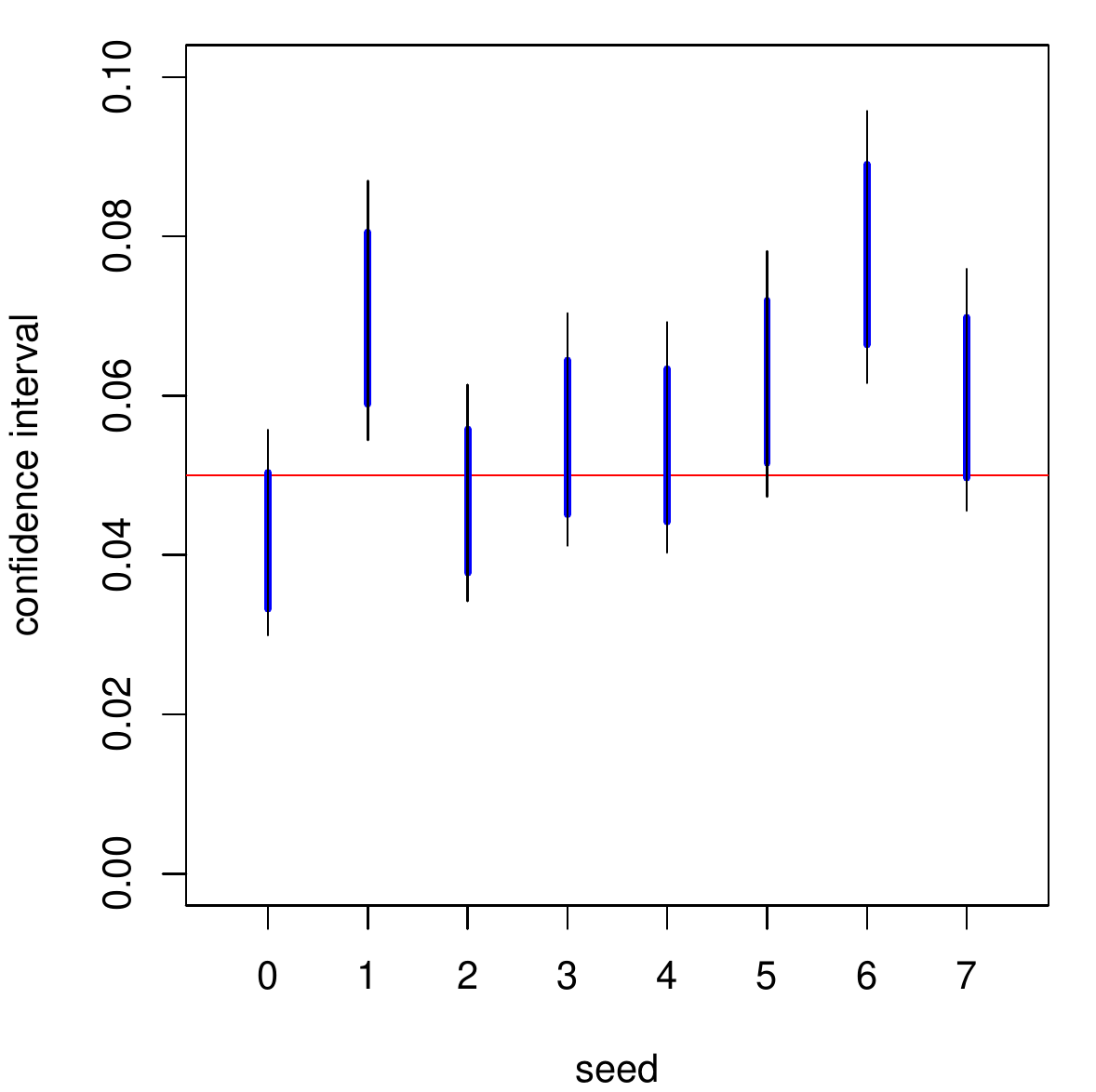}
\end{center}

\vspace{-7mm}

\caption{Confidence intervals for training conditional error probabilities:
  $95\%$ in black (thin lines) and $80\%$ in blue (thick lines).
  The $5\%$ significance level is shown as the horizontal red line.}
\label{fig:train_cond_exper}
\end{figure}

\begin{figure}[tb]
\begin{center}
  \includegraphics[width=0.4\columnwidth]{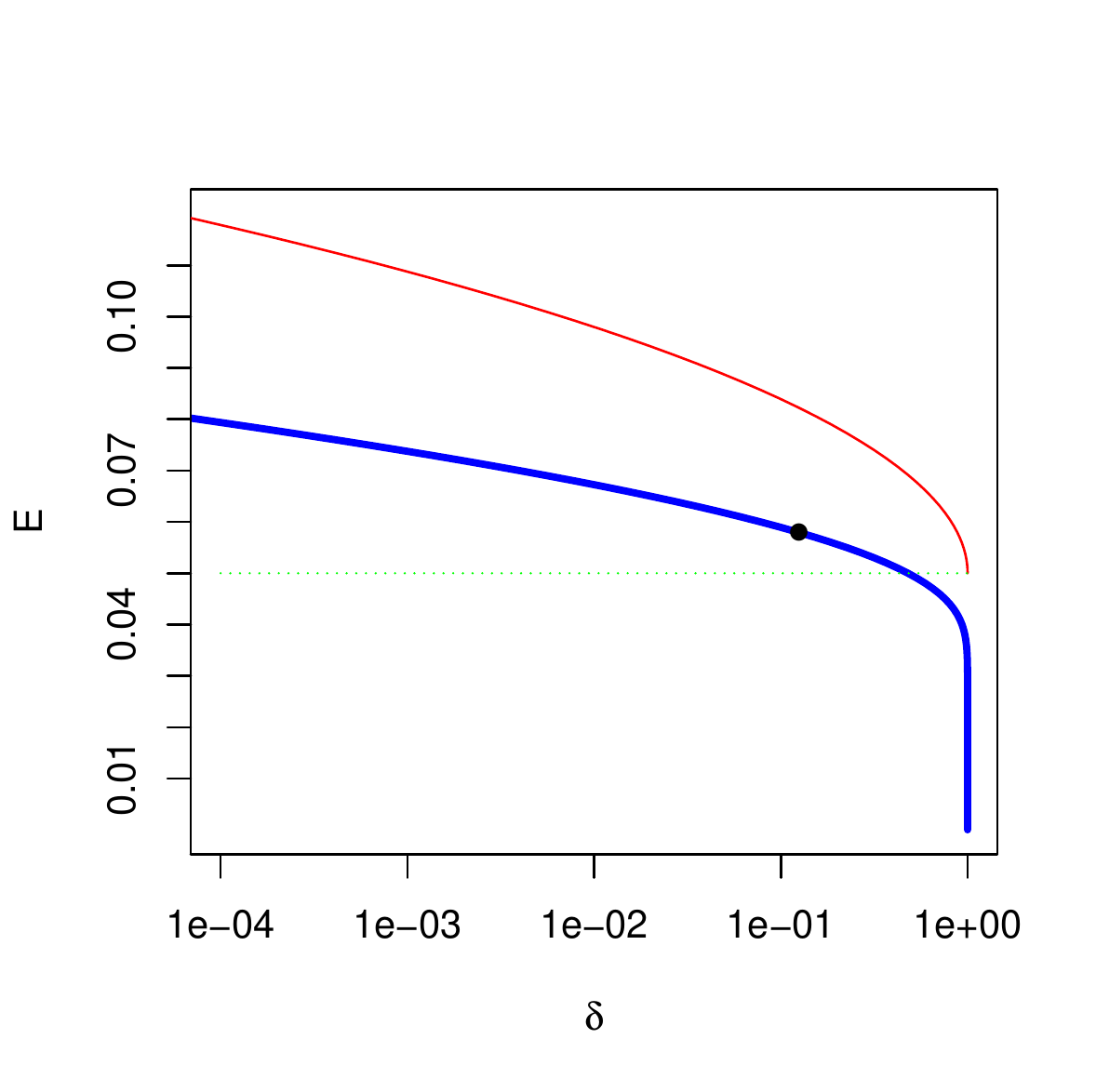}
\end{center}

\vspace{-10mm}

\caption{The probability of error $E$ vs $\delta$
from Propositions~\ref{prop:2-parameter-exact} (the thick blue line)
and ~\ref{prop:2-parameter} (the thin red line), where $\epsilon=0.05$ and $n=999$.}
\label{fig:train_cond_theor}
\end{figure}

\section{ICPs and ROC curves}
\label{sec:ROC}

This section will discuss a close connection between an important class of ICPs
(``probability-type'' label conditional ICPs) and ROC curves.
(For a previous study of connection between conformal prediction and ROC curves,
see \ifnotCONF\citealt{vanderlooy/sk:2007}\fi.)
Let us say that an ICP or a label conditional ICP is \emph{probability-type} if its inductive conformity measure
is defined by (\ref{eq:score}) where $f$ takes values in $\mathbb{R}$
and $\Delta$ is defined by (\ref{eq:Delta}).

The reader might have noticed
that the two leftmost plots in Figure~\ref{fig:scatter} look similar to a ROC curve.
The following proposition will show that this is not coincidental in the case of the lower left one.
However, before we state it, we need a few definitions.
We will now consider a general binary classification problem
and will denote the labels
as 0 and 1.
For a threshold $c\in\mathbb{R}$,
the \emph{type I error on the calibration set} is
\begin{equation}\label{eq:I}
  \alpha(c)
  :=
  \frac
  {\{i=m+1,\ldots,l\mid f(x_i)\ge c\;\&\;y_i=0\}}
  {\{i=m+1,\ldots,l\mid y_i=0\}}
\end{equation}
and the \emph{type II error on the calibration set} is
\begin{equation}\label{eq:II}
  \beta(c)
  :=
  \frac
  {\{i=m+1,\ldots,l\mid f(x_i)\le c\;\&\;y_i=1\}}
  {\{i=m+1,\ldots,l\mid y_i=1\}}
\end{equation}
(with $0/0$ set, e.g., to $1/2$).
Intuitively,
these are the error rates for the classifier that predicts $1$ when $f(x)>c$ and predicts $0$ when $f(x)<c$;
our definition is conservative in that it counts the prediction as error whenever $f(x)=c$.
The \emph{ROC curve} is the parametric curve
\begin{equation}\label{eq:ROC}
  \{(\alpha(c),\beta(c))\mid c\in\mathbb{R}\}
  \subseteq
  [0,1]^2.
\end{equation}
(Our version of ROC curves is the original version reflected in the line $y=1/2$;
in our sloppy terminology we follow \citealt{hastie/etal:2009},
whose version is the original one reflected in the line $x=1/2$,
and many other books and papers\ifnotCONF;
  see, e.g., \citealt{bengio/etal:2005}, Figure~1\fi.)

\begin{proposition}\label{prop:ROC}
  In the case of a probability-type label conditional ICP,
  for any object $x\in\mathbf{X}$,
  the distance between the pair $(p^0,p^1)$ (see (\ref{eq:p-cond}))
  and the ROC curve is at most
  \begin{equation}\label{eq:distance}
    \sqrt
    {
      \frac{1}{(n^0+1)^2}
      +
      \frac{1}{(n^1+1)^2}
    },
  \end{equation}
  where $n^y$ is the number of examples in the calibration set labelled as $y$.
\end{proposition}
\begin{proof}
  Let $c:=f(x)$.
  Then we have
  \begin{equation}\label{eq:p-p}
    (p^0,p^1)
    =
    \left(
      \frac{n^0_{\ge}+1}{n^0+1},
      \frac{n^1_{\le}+1}{n^1+1}
    \right)
  \end{equation}
  where $n^0_{\ge}$ is the number of examples $(x_i,y_i)$ in the calibration set such that $y_i=0$ and $f(x_i)\ge c$
  and $n^1_{\le}$ is the number of examples in the calibration set such that $y_i=1$ and $f(x_i)\le c$.
  It remains to notice that the point
  $
    \left(
      n^0_{\ge}/n^0,
      n^1_{\le}/n^1
    \right)
  $
  belongs to the ROC curve:
  the horizontal (resp.\ vertical) distance between this point and (\ref{eq:p-p})
  does not exceed $1/(n^0+1)$ (resp.\ $1/(n^1+1)$),
  and the overall Euclidean distance does not exceed (\ref{eq:distance}).
\end{proof}

So far we have discussed the \emph{empirical ROC curve}:
(\ref{eq:I}) and (\ref{eq:II}) are the empirical probabilities of errors of the two types
on the calibration set.
It corresponds to the estimate $k/n$ of the parameter of the binomial distribution
based on observing $k$ successes out of $n$.
The minimax estimate is $(k+1/2)/(n+1)$,
and the corresponding ROC curve (\ref{eq:ROC}) where $\alpha(c)$ and $\beta(c)$
are defined by (\ref{eq:I}) and (\ref{eq:II}) with the numerators increased by $\frac12$
and the denominators increased by $1$
will be called the \emph{minimax ROC curve}.
Notice that for the minimax ROC curve we can put a coefficient of $\frac12$ in front of (\ref{eq:distance}).
Similarly, when using the Laplace estimate $(k+1)/(n+2)$,
we obtain the \emph{Laplace ROC curve}.
See Figure~\ref{fig:scatter_ROC} for the lower left corner of the lower left plot of Figure~\ref{fig:scatter}
with different ROC curves added to it.

\begin{figure}[tb]
\begin{center}
  \includegraphics[width=0.4\columnwidth]{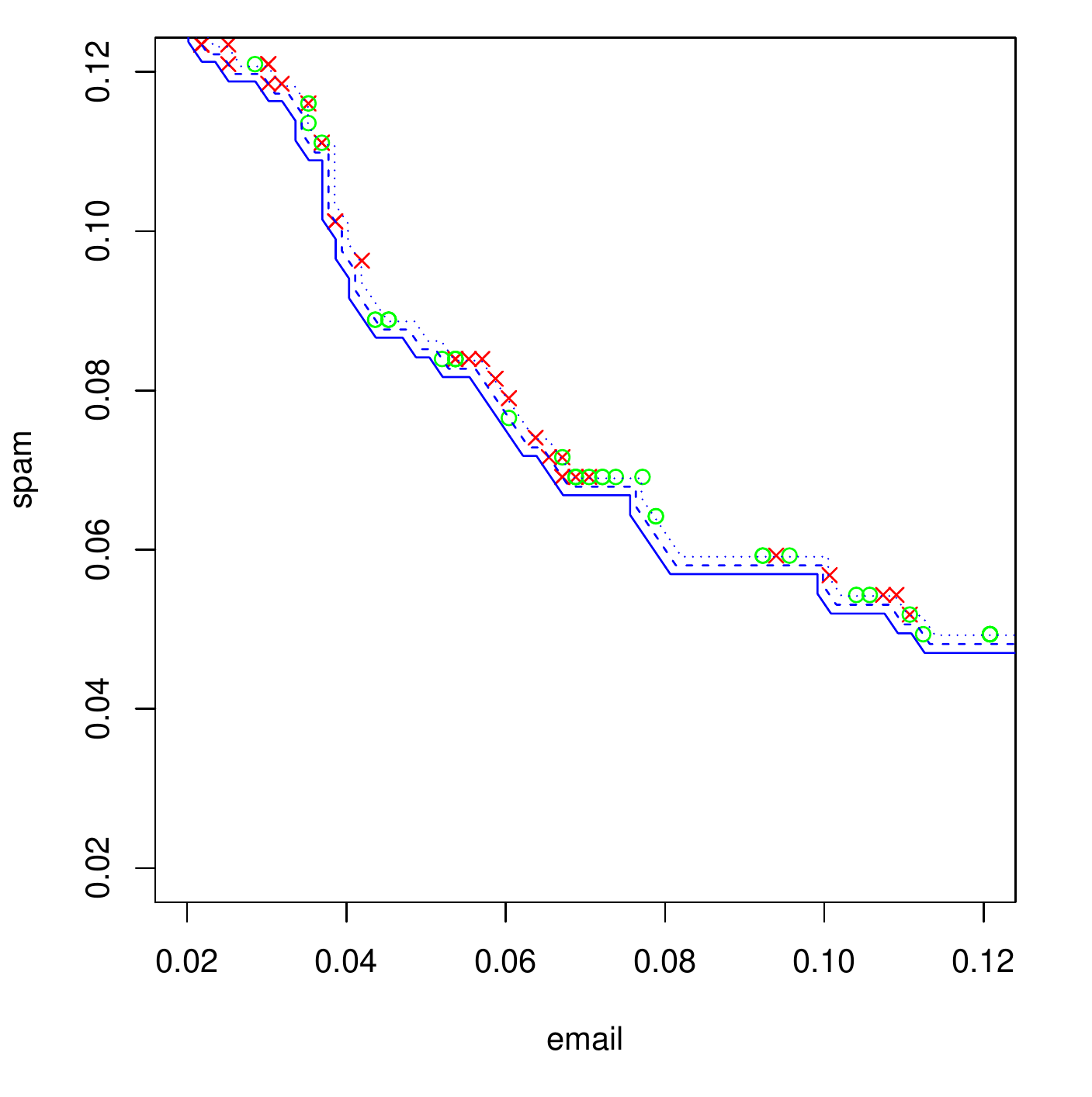}
\end{center}

\vspace{-8mm}

\caption{The lower left corner of the lower left plot of Figure~\ref{fig:scatter}
with the empirical (solid blue), minimax (dashed blue), and Laplace (dotted blue) ROC curves.}
\label{fig:scatter_ROC}
\end{figure}

In conclusion of our study of the \texttt{Spambase} data set,
we will discuss the asymmetry of the two kinds of error in spam detection:
classifying email as \texttt{spam} is much more harmful than letting occasional spam in.
A reasonable approach is to start from a small number $\epsilon>0$,
the maximum tolerable percentage of email classified as \texttt{spam},
and then to try to minimize the percentage of spam classified as \texttt{email} under this constraint.
The standard way of doing this is to classify a message $x$ as \texttt{spam} if and only if $f(x)\ge c$,
where $c$ is the point on the ROC curve corresponding to the type I error $\epsilon$.
It is not clear what this means precisely,
since we only have access to an estimate of the true ROC curve
(and even on the true ROC curve such a point might not exist).
But roughly,
this means classifying $x$ as \texttt{spam} if $f(x)$
exceeds the $k$th largest value in the set $\{\alpha_i\mid i\in\{m+1,\ldots,l\}\;\&\;y_i=\texttt{email}\}$,
where $k$ is close to $\epsilon n^0$ and $n^0$ is the size of this set
(i.e., the number of email in the calibration, or validation, set).
To make this more precise,
we can use the ``one-sided label conditional ICP'' classifying $x$ as \texttt{spam}
if and only if\ifnotCONF
  \footnote{In practice, we might want to improve the predictor
    by adding another step and changing the classification from \texttt{spam} to \texttt{email}
    if $p^1$ is also small,
    in which case $x$ looks neither like spam nor email.
    In view of Proposition~\ref{prop:ROC}, however, this step can be disregarded for probability-type ICP
    unless $\epsilon$ is very lax.}\fi\
$p^0\le\epsilon$ for $x$.
According to (\ref{eq:p-p}),
this means that we classify $x$ as \texttt{spam} if and only if $f(x)$
exceeds the $k$th largest value in the set $\{\alpha_i\mid i\in\{m+1,\ldots,l\}\;\&\;y_i=\texttt{email}\}$,
where $k:=\lfloor\epsilon(n^0+1)\rfloor$.
The advantage of this version of the standard method
is that it guarantees that the probability of mistaking email for spam is at most $\epsilon$
(see Proposition~\ref{prop:validity-ICP-cond})
and also enjoys the training conditional version of this property given by Proposition~\ref{prop:2-parameter}
(more accurately, its version for label conditional ICPs).

\section{Conclusion}
\label{sec:conclusion}

The goal of this paper has been to explore various versions of the requirement of conditional validity.
With a small training set, we have to content ourselves with unconditional validity
(or abandon any formal  requirement of validity altogether).
For bigger training sets training conditional validity will be approached by ICPs automatically,
and we can approach example conditional validity by using conditional ICPs
but making sure that the size of a typical category does not become too small (say, less than 100).
In problems of binary classification,
we can control false positive and false negative rates by using label conditional ICPs.

The known property of validity of inductive conformal predictors (Proposition~\ref{prop:validity-ICP})
can be stated in the traditional statistical language
(see, e.g., \citealt{fraser:1957}\ifnotCONF\
  and \citealt{guttman:1970}\fi)
by saying that they are $1-\epsilon$ expectation tolerance regions,
where $\epsilon$ is the significance level.
In classical statistics, however, there are two kinds of tolerance regions:
$1-\epsilon$ expectation tolerance regions
and PAC-type $1-\delta$ tolerance regions for a proportion $1-\epsilon$,
in the terminology of \citet{fraser:1957}.
We have seen (Proposition~\ref{prop:2-parameter}) that inductive conformal predictors
are tolerance regions in the second sense as well
(cf.\ 
Appendix~A).

A disadvantage of inductive conformal predictors is their potential predictive inefficiency:
indeed, the calibration set is wasted
as far as the development of the prediction rule $f$ in (\ref{eq:score}) is concerned,
and the proper training set is wasted
as far as the calibration (\ref{eq:p}) of conformity scores into p-values is concerned.
Conformal predictors use the full training set for both purposes,
and so can be expected to be significantly more efficient.
(There have been reports
of comparable and even better predictive efficiency of ICPs as compared to conformal predictors
but they may be unusual artefacts of the methods used and particular data sets.)
It is an open question whether we can guarantee training conditional validity
under (\ref{eq:E}) or a similar condition for conformal predictors
different from classical tolerance regions.
Perhaps no universal results of this kind exist,
and different families of conformal predictors will require different methods.



\ifnotCONF
  \subsection*{Acknowledgments}
\fi

The empirical studies described in this paper used the R system 
and the \texttt{gbm} package written by Greg Ridgeway\ifnotCONF\
(based on the work of \citealt{freund/schapire:1997} and \citealt{friedman:2001,friedman:2002})\fi.
This work was partially supported by the Cyprus Research Promotion Foundation.
Many thanks to the reviewers \ifnotCONF of the conference version of the paper \fi
for their advice.

\ifnotCONF
\appendix
\section{Training conditional validity for classical tolerance regions}
\label{app:CP}

In this appendix we compare Propositions~\ref{prop:2-parameter} and~\ref{prop:2-parameter-exact}
with the results
(see, e.g., \citealt{fraser:1957} and \citealt{guttman:1970})
about classical tolerance regions
(which are a special case of conformal predictors,
as explained in \citealt{vovk/etal:2005book}, p.~257).
It is well known that under appropriate continuity assumptions the classical tolerance regions
that discard $\epsilon(n+1)$ out of the $n+1$ statistically equivalent blocks
(in this appendix we always assume that $\epsilon(n+1)$ is an integer number)
have coverage probability following the beta distribution
with parameters $(1-\epsilon)(n+1)$ and $\epsilon(n+1)$
(see, e.g., \citealt{tukey:1947} or \citealt{guttman:1970}, Theorems 2.2 and 2.3);
in particular, their expected coverage probability is $1-\epsilon$.
This immediately implies the following corollary:
  if $\Gamma$ is a classical tolerance predictor with sample size $n$ and expected coverage probability $1-\epsilon$,
  it is $(E,\delta)$-valid if and only if
  \begin{equation}\label{eq:E-classical}
    \delta
    \ge
    \Bet_{(1-\epsilon)(n+1),\epsilon(n+1)}(1-E)
    =
    1-\Bet_{\epsilon(n+1),(1-\epsilon)(n+1)}(E),
  \end{equation}
  where $\Bet_{\alpha,\beta}$ is the cumulative beta distribution function
  with parameters $\alpha$ and $\beta$.


The following lemma shows that in fact (\ref{eq:E-classical}) coincides
with the condition (\ref{eq:E-exact}) for ICPs
(under our assumption $\epsilon(n+1)\in\mathbb{Z}$).
Of course, $n$ means different things in (\ref{eq:E-exact}) and (\ref{eq:E-classical}):
the size of the calibration set in the former and the size of the full training set
in the latter.

\begin{lemma}[\url{http://dlmf.nist.gov/8.17.E5}]
  For all $n\in\{1,2,\ldots\}$, all $k\in\{0,1,\ldots,n\}$, and all $E\in[0,1]$,
  \begin{equation}\label{eq:equality}
    \bin_{n,E}(k-1)
    =
    \Bet_{n+1-k,k}(1-E)
    =
    1-\Bet_{k,n+1-k}(E).
  \end{equation}
\end{lemma}
\nocite{DLMF}

\begin{proof}
  The equality between the last two terms of (\ref{eq:equality}) is obvious.
  The last term of (\ref{eq:equality}) is the probability
  that the $k$th smallest value in a sample of size $n$ from the uniform probability distribution $U$ on $[0,1]$
  exceeds $E$.
  This event is equivalent to at most $k-1$ of $n$ independent random variables generated from $U$
  belonging to the interval $[0,E]$,
  and so the probability of this event is given by the first term of (\ref{eq:equality}).
\end{proof}

The assumption of continuity was removed by \citet{tukey:1948} and \citet{fraser/wormleighton:1951}.
We will state this result only for the simplest kind of classical tolerance regions,
essentially those introduced by \citet{wilks:1941}
(this special case was obtained already by \citealt{scheffe/tukey:1945}, p.~192).
Suppose the object space $\mathbf{X}$ is a one-element set
and the label space is $\mathbf{Y}=\mathbb{R}$
(therefore, we consider the problem of predicting real numbers without objects).
For two numbers $L\le U$ in the set $\{0,1,\ldots,n+1\}$ consider the set predictor
$[y_{(L)},y_{(U)}]$, where $y_{(i)}$ is the $i$th order statistics
(the $i$th smallest value in the training set $(y_1,\ldots,y_n)$,
except that $y_{(0)}:=-\infty$ and $y_{(n+1)}:=\infty$).
This set predictor is $(E,\delta)$-valid provided we have (\ref{eq:E-classical})
with $\epsilon(n+1)$ replaced by $L+n+1-U$.

It is easy to see that Proposition~\ref{prop:2-parameter-exact}
(and, therefore, Proposition~\ref{prop:2-parameter})
can in fact be deduced from Scheff\'e and Tukey's result.
This follows from the interpretation of inductive conformal predictors
as a ``conditional'' version of Wilks's predictors corresponding to $L:=\epsilon(n+1)$ and $U:=n+1$.
After observing the proper training set we apply Wilks's predictors
to the conformity scores $\alpha_i$ of the calibration examples
to predict the conformity score of a test example;
the set prediction of the conformity score for the test object
is transformed into the prediction set consisting of the labels leading to a score in the predicted range.
\fi

\end{document}